\documentclass[sigconf]{acmart}

\pdfoutput=1
\usepackage{hyperref}       
\usepackage{float}
\usepackage{booktabs} 
\usepackage{amsmath}
\usepackage{dsfont}
\usepackage{commath}
\usepackage{amsthm}
\usepackage{enumitem}
\usepackage{color}

\newtheorem{theorem}{Theorem}[section]

\newtheorem{corollary}[theorem]{Corollary}
\newcommand{\etal}{\textit{et al}. }
\newcommand{\ie}{\textit{i}.\textit{e}.}
\newcommand{\eg}{\textit{e}.\textit{g}.}
\usepackage{balance}

\newenvironment{definition}[1][Definition]{\begin{trivlist}
\item[\hskip \labelsep {\bfseries #1}]}{\end{trivlist}}

\newenvironment{claim}[1][Claim]{\begin{trivlist}
\item[\hskip \labelsep {\bfseries #1}]}{\end{trivlist}}

\usepackage{caption}
\usepackage{subcaption}
\usepackage[utf8]{inputenc} 
\usepackage[T1]{fontenc}    
\usepackage{url}            
\usepackage{booktabs}       
\usepackage{amsfonts}       
\usepackage{nicefrac}       
\usepackage{microtype}      

\usepackage{tablefootnote}
\usepackage[flushleft]{threeparttable}

\copyrightyear{2018} 
\acmYear{2018} 
\setcopyright{acmcopyright}
\acmConference[KDD '18]{The 24th ACM SIGKDD International Conference on Knowledge Discovery and Data Mining}{August 19--23, 2018}{London, United Kingdom}
\acmBooktitle{KDD '18: The 24th ACM SIGKDD International Conference on Knowledge Discovery and Data Mining, August 19--23, 2018, London, United Kingdom}
\acmPrice{15.00}
\acmDOI{10.1145/3219819.3220070}
\acmISBN{978-1-4503-5552-0/18/08}

\begin{document}
\fancyhead{}
\title{Learning Credible Models}






\author{Jiaxuan Wang}
\affiliation{%
  \institution{University of Michigan}
}
\email{jiaxuan@umich.edu}

\author{Jeeheh Oh}
\affiliation{%
  \institution{University of Michigan}
}
\email{jeeheh@umich.edu}

\author{Haozhu Wang}
\affiliation{%
  \institution{University of Michigan}
}
\email{hzwang@umich.edu}

\author{Jenna Wiens}
\affiliation{%
  \institution{University of Michigan}
}
\email{wiensj@umich.edu}

\renewcommand{\shortauthors}{J. Wang et al.}

\begin{abstract}
In many settings, it is important that a model be capable of providing reasons for its predictions (\ie, the model must be interpretable). However, the model's reasoning may not conform with well-established knowledge. In such cases, while interpretable, the model lacks \textit{credibility}. In this work, we formally define credibility in the linear setting and focus on techniques for learning models that are both accurate and credible. In particular, we propose a regularization penalty, expert yielded estimates (EYE), that incorporates expert knowledge about well-known relationships among covariates and the outcome of interest. We give both theoretical and empirical results comparing our proposed method to several other regularization techniques. Across a range of settings, experiments on both synthetic and real data show that models learned using the EYE penalty are significantly more credible than those learned using other penalties. Applied to two large-scale patient risk stratification task, our proposed technique results in a model whose top features overlap significantly with known clinical risk factors, while still achieving good predictive performance.
\end{abstract}

\keywords{Model Interpretability, Regularization}

\settopmatter{printfolios=true}

\maketitle

\section{Introduction}

\begin{table*}[t] 
  \caption{A comparison of relevant regularization penalties.}
  \centering
  \scalebox{0.8}{
  \begin{tabular}{lcccl}
    \toprule
    Method & Formulation & Sparsity & Grouping effect & Consistency \\
    \midrule
  LASSO & $\Vert \boldsymbol{\theta} \Vert_1$ & yes & no & conditioned \cite{zhao2006model}\\
  ridge  & $\frac{1}{2}\Vert \boldsymbol{\theta} \Vert_2^2$ & no & yes & yes\\
  elastic net & $\beta \Vert \boldsymbol{\theta} \Vert_1 + \frac{1}{2} (1-\beta) \Vert \boldsymbol{\theta} \Vert_2^2$ & yes & yes & conditioned \cite{jia2010model}\\
  OWL & $\sum_{i=1}^n w_i |\boldsymbol{\theta}|_{[i]}$ & yes & yes & unknown \\
  weighted LASSO & $\Vert \boldsymbol{w} \odot \boldsymbol{\theta}\Vert_1$ & yes & yes & no\\
  weighted ridge & $\frac{1}{2}\Vert \boldsymbol{w} \odot \boldsymbol{\theta} \Vert_2^2$ & no & no & no \\
    adaptive LASSO &$\Vert \boldsymbol{w^*} \odot \boldsymbol{\theta}\Vert_1$ & yes & no & conditioned \cite{zou2006adaptive} \\
    \bottomrule
  \end{tabular}}
  \label{table3}
\end{table*}

 For adoption, predictive models must achieve good predictive performance. Often, however, good performance alone is not enough. In many settings, the model must also be interpretable or capable of providing reasons for its predictions. For example, in healthcare applications, research has shown that decision trees are preferred among physicians because of their high level of interpretability \cite{meyfroidt2009machine,kononenko2001machine}. Still, interpretability alone may not be enough to encourage adoption. If the reasons provided by the model do not agree, at least in part, with well-established domain knowledge, practitioners may be less likely to trust and adopt the model. 

 Often, one ends up trading off such credibility for interpretability, especially when it comes to learning sparse models. For example, regularization penalties, like the LASSO penalty, encourage sparsity in the learned feature weights, but in doing so may end up selecting features that are merely associated with the outcome rather than those that are known to affect the outcome. This can easily occur when there is a high-degree of collinearity present in one’s data. In short, interpretability does not imply credibility.

 Informally, a credible model is an interpretable model that i) provides reasons for its predictions that are, at least in part, inline with well-established domain knowledge, and ii) does no worse than other models in terms of predictive performance. While a user is more likely to adopt a model that agrees with well-established domain knowledge, one should not have to sacrifice accuracy to achieve such adoption. That is, the model should only agree with well-established knowledge, if it is consistent with the data. Relying on domain expertise alone would defeat the purpose of data-driven algorithms, and could result in worse performance. Admittedly, the definition of credibility is a subjective matter. In this work, we offer a first attempt to formalize the intuition behind a credible model. 

Our main contributions include:
\begin{itemize}
  \item formally defining credibility in the linear setting
  \item proposing a novel regularization term EYE (expert yielded estimates) to achieve this form of credibility. 
\end{itemize}
 Our proposed approach leverages domain expertise regarding known relationships between the set of covariates and the outcome. This domain expertise is used to guide the model in selecting among highly correlated features, while encouraging sparsity. Our proposed framework allows for a form of collaboration between the data-driven learning algorithm and the expert. We prove desirable properties of our approach in the least squares regression setting. Furthermore, we give empirical evidence of these properties on synthetic and real datasets. Applied to two large-scale patient risk stratification tasks, our proposed approach resulted in an accurate model and a feature ranking that, when compared to a set of well-established risk factors, yielded an average precision (AP) an order of magnitude greater than the second most credible model in one task, and twice as large in AP in the other task.

 The rest of the paper is organized as follows. Section \ref{related_work} reviews related work on variable selection and interpretability. Section \ref{proposed_approach} defines credibility and describes our proposed method in detail. Section \ref{experiments} presents experiments and results. Section \ref{discussion} summarizes the importance of our work and suggests potential extensions of our proposed method.

\section{Related Work} \label{related_work}

Credibility is closely related to interpretability, which has been actively explored in the literature \cite{hara2016finding,lakkaraju2017learning,lipton2016mythos,ribeiro2016should,vstrumbelj2014explaining,ustun2014methods}. Yet, to the best of our knowledge, credibility has never been formally studied. 

Interpretability is often achieved through dimensionality reduction. Common approaches include preprocessing the data to eliminate correlation, or embedding a feature selection criterion into the model's objective function. Embedding a regularization term in the objective function is often preferred over preprocessing techniques since it is nonintrusive in the training pipeline. Thus, while credible models could, in theory, be achieved by first preprocessing the data, we focus on a more general approach that relies on regularization. 

The most common forms of regularization, $l_1$ (LASSO) and $l_2$ (ridge), can be interpreted as placing a prior distribution on feature weights \cite{zou2006adaptive} and can be solved analytically (LASSO in the orthogonal case, ridge in the general setting). The sparsity in feature weights induced by LASSO's diamond shaped contour is often desirable, thus many extensions of it have been proposed, including elastic net \cite{zou2005regularization}, ordered weighted LASSO (OWL) \cite{owl}, adaptive LASSO \cite{zou2006adaptive}, and weighted LASSO \cite{bergersen2011weighted}. 

In Table \ref{table3}, we summarize relevant properties for several common regularization terms. $\boldsymbol{\theta}$ represents the model parameters; $\beta \in [0,1]$ is a hyperparameter that controls the tradeoff between the $l_1$ and $l_2$ norms; $\boldsymbol{w}$ is a set of non-negative weights for each feature; $\boldsymbol{w^*}$ is the optimal set of weights (according to a least squares solution) \cite{zou2006adaptive}; $|\boldsymbol{\theta}|_{[i]}$ is the $i^{th}$ largest parameter sorted by magnitude; and $\odot$ is the elementwise product. The grouping effect refers to the ability to group highly correlated covariates together \cite{zou2005regularization}, and consistency refers to the property that learned features converge in distribution to the true underlying feature weights \cite{ibragimov2013statistical}. Without the grouping effect, some relevant features identified as important by experts may end up not being selected because they are correlated with other relevant expert recommended features.

In terms of incorporating additional expert knowledge at training time, Sun \etal explore using features identified as relevant during training, along with a subset of other features that yield the greatest improvement in predictive performance \cite{sun2012combining}. This work differs from ours because they assume expert knowledge as ground truth, a potentially dangerous assumption when experts are wrong. Vapnik \etal explore the theory of learning with privileged information \cite{vapnik2015learning}. Though similar in setting, they use expert knowledge to accelerate the learning process, not to enforce credibility.  Helleputte and Dupont use partially supervised approximation of zero-norm minimization (psAROM) to create a sparse set of relevant features. Much like weighted LASSO, psAROM does not exhibit the grouping effect, thus is unable to retain all known relevant features. Moreover, the non-convex objective function for psAROM makes exact optimization hard \cite{helleputte2009partially}. \cite{choi2017gram} looks at utilizing hierarchical expert information to learn embeddings that help model prediction of rare diseases. While it is an interesting approach, its model's interpretability is questionable. \cite{ross2017right} constrains the input gradient of features that are believed not to be relevant in a neural network. In the linear setting, the method simplifies to $l_2$ regularization on unknown features, which is suboptimal for model interpretability because the learned weights are dense.

Perhaps closest to our proposed approach, and the concept of credibility, is related work in interpretability that focuses on enforcing monotonicity constraints between the covariates and the prediction \cite{ben1995monotonicity,kotlowski2009rule,martens2011performance,pazzani2001acceptance,verbeke2011building}. The main idea behind this branch of work is to restrict classifiers to the set of monotone functions. This restriction could be probabilistic \cite{kotlowski2009rule} or monotone in certain arguments identified by experts \cite{ben1995monotonicity,pazzani2001acceptance,verbeke2011building}. Though similar in aim (having models inline with domain expertise), previous work has focused on rule based systems. Other attempts to enforce monotonicity in nonlinear models \cite{altendorf2012learning,sill1998monotonic,velikova2006solving} aim to increase performance. Again, relying too heavily on expert knowledge may result in a decrease in performance when experts are wrong. In contrast, we propose a general regularization technique that aims to increase credibility without decreasing performance. Moreover, in the linear setting, credible models satisfy monotonicity and sparsity constraints.

\section{Proposed Approach} \label{proposed_approach}

In this paper, we focus on linear models. Within this setting, we start by formally defining credibility in \ref{definition}. Then, building off of a naïve approach in \ref{naive_approach}, we introduce our proposed approach in \ref{EYE}. In \ref{properties}, we state important properties and theoretical results relevant to our proposed method. 

\subsection{Definition and Notation} \label{definition}

Interpretability is a prerequisite for credibility. For linear models, interpretability is often defined as sparsity in the feature weights. Here, we define the set of features as $\mathcal{D}$. We assume that we have some domain expertise that identifies $\mathcal{K}\subseteq\mathcal{D}$, a subset of the features as known (or believed) to be important. Intuitively, among a group correlated features a credible model will select those in $\mathcal{K}$, if the relationship is consistent with the data.

Consider the following unconstrained empirical risk minimization problem, 
$\boldsymbol{\hat \theta} =  \arg\min_{\boldsymbol{\theta}} L(\boldsymbol{\theta}, X, \boldsymbol{y}) + n \lambda J(\boldsymbol{\theta}, \boldsymbol{r})$ that minimizes the sum of some loss function $L$ and regularization term $J$. $X$ is an $n$ by $d$ design matrix, where row $\boldsymbol{x}$ corresponds to one observation. The corresponding entry in $\boldsymbol{y} \in \mathbb{R}^n$ is the target value for $\boldsymbol{x}$. Let $v_i$ denote the $i^{th}$ entry of a vector $\boldsymbol{v}$. {\small $\lambda \in \mathbb{R}_{\geq 0}$} is the tradeoff between loss and regularization, and $\boldsymbol{r} \in \{0,1\}^d$ is the indicator array where $r_i=1$ if {\small $i\in\mathcal{K}$} and $0$ otherwise. Note that our setting differs from the conventional setting only through the inclusion of $\boldsymbol{r}$ in the regularization term. For theoretical convenience, we prove theorems in the least squares regression setting and denote {\small $\boldsymbol{\hat \theta}^{OLS}$} as the ordinary least squares solution. For experiments, we use logistic loss. 

We denote $\boldsymbol{\theta}$ as the true underlying parameters. Then $\boldsymbol{\theta}_{\mathcal{K}}$ and $\boldsymbol{\theta}_{\mathcal{D}\setminus \mathcal{K}}$ are the true parameters associated with the subset of known and unknown features, respectively. Throughout the text, vectors are in bold, and estimates are denoted with a hat. 

\begin{definition}
A linear model is \emph{credible} if

\begin{enumerate}[label=(\roman*)]
\itemsep0em
\item Within a group of correlated \textit{relevant} features $\mathcal{C}\subseteq{D}$: $\boldsymbol{\hat\theta}_{\mathcal{K}\cap \mathcal{C}}$ is dense, and $\boldsymbol{\hat\theta}_{\mathcal{C}\setminus \mathcal{K}}$ is sparse (\textit{structure constraint}).
\item  Model performance is comparable with other regularization techniques (\textit{performance constraint})
\end{enumerate}

Consider the following toy example where $|\mathcal{C}|=2$ and one of these features has been identified $\in \mathcal{K}$ by the expert, while the other has not. One could arbitrarily select among these two correlated features, including only one in the model. To increase credibility, we encourage the model to select the known feature (\ie, the feature in $\mathcal{K}$)

We stress \textit{relevant} in the definition because we do not care about the structure constraint if the group of variables does not contribute to the predictive performance. We assume expert knowledge is sparse compared to all features; thus a credible model is sparse due to the structure requirement. Credible models will result in dense weights among the known features, if the expert knowledge provided is indeed supported by the data. If experts are incorrect, \ie, the set of features $\mathcal{K}$ are not relevant to the task at hand, then credible models will discard these variables, encouraging sparsity.


\end{definition}

\subsection{A Naïve Approach to Credibility} \label{naive_approach}

Intuitively, one may achieve credibility by constraining weights for known important factors with the $l_2$ norm and weights for other features with the $l_1$ norm. The $l_2$ norm will maintain a dense structure in known important factors and the $l_1$ norm will encourage sparsity on all remaining covariates. Formally, this penalty can be written as  $q(\boldsymbol{\theta}) = (1-\beta) \Vert \boldsymbol{r} \odot \boldsymbol{\theta} \Vert_2^2 + 2 \beta  \Vert (\boldsymbol{1}-\boldsymbol{r}) \odot \boldsymbol{\theta} \Vert_1$ where $\boldsymbol{\theta} \in \mathbb{R}^d$, $\beta \in (0,1)$ controls the tradeoff between weights associated with the features in $\mathcal{K}$ and in $\mathcal{D} \setminus \mathcal{K}$. 

Unfortunately, $q$ does not encourage sparsity in {\small $\boldsymbol{\hat\theta}_{\mathcal{D}\setminus \mathcal{K}}$}. \textbf{Figure \ref{fig:sub1}} shows its contour plot. For a convex problem, each level set of the contour corresponds to a feasible region associated with a particular $\lambda$. A larger level value implies a smaller $\lambda$. It is clear from the figure that this penalty is non-homogeneous, that is $f(t\boldsymbol{x}) \neq |t|f(\boldsymbol{x})$. In a two-dimensional setting, when the covariates perfectly correlate with one another, the level curve for the loss function will have a slope of $-1$ corresponding to the violet dashed lines in \textbf{Figure \ref{fig:fig1}}. 

To understand why the slope must be $-1$, consider the classifier $y=\theta_{\mathcal{K}} x_1 + \theta_{\mathcal{D}\setminus\mathcal{K}} x_2$. Since $x_1$ and $x_2$ are perfectly correlated by assumption, we have $y=(\theta_{\mathcal{K}} + \theta_{\mathcal{D}\setminus\mathcal{K}}) x_1$. Note that the loss value is fixed as long as $\theta_{\mathcal{K}}+\theta_{\mathcal{D}\setminus\mathcal{K}}$ is fixed, which means that each level curve of the loss function has the form $\theta_{\mathcal{K}}+\theta_{\mathcal{D}\setminus\mathcal{K}}=c$ for some scaler $c$, \ie, $\theta_{\mathcal{D}\setminus\mathcal{K}}=-\theta_{\mathcal{K}}+c$. Thus, the slope of the violet lines must be $-1$ in \textbf{Figure \ref{fig:fig1}}.

By the KKT conditions, with $\lambda>0$, the optimal solution (red dots for each level curve in \textbf{Figure \ref{fig:fig1}}) occurs at the boundary of the contour with the same slope ($\lambda=0$ means the problem is unconstrained, then all methods are equal). We observe that with small $\lambda$, the large constraint region forces the model to favor features not in $\mathcal{K}$ because the point on the boundary with slope of $-1$ occurs near $\theta_{\mathcal{D}\setminus \mathcal{K}}$ axis, leading to a model that is not credible. 

\subsection{The Expert Yielded Estimates (EYE) Penalty} \label{EYE}

To address this sensitivity to the choice of hyperparameter, we propose the EYE penalty, obtained by fixing a level curve of $q$ and scaling it for different contour levels. The trick is to force the slope of level curve in the positive quadrant to approach $-1$ as {\small $\theta_{\mathcal{D}\setminus \mathcal{K}}$} approaches $0$. Note that since $q$ is symmetric around both axes, we can just focus on one "corner." That is, we want the "corner" on the right of the level curve to have a slope of $-1$, so that {\small $\boldsymbol{\hat \theta}$} hits it in the perfectly correlated case. In fact, as long as $-1 \leq$ the "corner" slope $\leq 0$, we achieve the desired feature selection. In the extreme case of slope $0$ ($\beta=1$), we do not penalize $\boldsymbol{\theta}_{\mathcal{K}}$ at all. Using a slope with a magnitude smaller than $1$ assumes that features in $\mathcal{K}$ are much more relevant than other features, thus biasing $\boldsymbol{\hat \theta}_{\mathcal{K}}$. Since we do not wish to bias $\boldsymbol{\hat\theta}_{\mathcal{K}}$ towards larger values, if the solution is inconsistent with the data, we keep the slope as $-1$. This minimizes the effect of our potential prejudices, while maintaining the desirable feature selection properties. Casting our intuition mathematically yields the EYE penalty:
{\small
\begin{equation} \label{eye-defn-orig}
 eye(\boldsymbol{x}) = \inf \left \{t>0  \mid \boldsymbol{x} \in \left \{t \boldsymbol{x} \mid q(\boldsymbol{x}) \leq \frac{\beta^2}{1-\beta} \right \} \right \}
\end{equation}}
where $t$ is a scaling factor to make EYE homogeneous and the inner set defines the level curve to fix. Note that $\beta$ only scales the EYE penalty, thus can rewrite the penalty as:
\begin{equation} \label{eye-defn}
   eye(\boldsymbol{\theta}) = \Vert (\boldsymbol{1}-\boldsymbol{r}) \odot \boldsymbol{\theta} \Vert_1 + \sqrt{ \Vert (\boldsymbol{1}-\boldsymbol{r}) \odot \boldsymbol{\theta} \Vert_1^2   +  \Vert \boldsymbol{r} \odot \boldsymbol{\theta} \Vert_2^2} 
\end{equation}
Derivations of (\ref{eye-defn-orig}) and (\ref{eye-defn}) are included in the Appendix. \textbf{Figure \ref{fig:sub2}} shows the contour plot of EYE penalty (note that the optimal solution for each level set occurs at the "corner" as desired). 

\begin{figure*}
\centering
\begin{subfigure}{0.33\textwidth}
  \centering
  \includegraphics[width=1\linewidth]{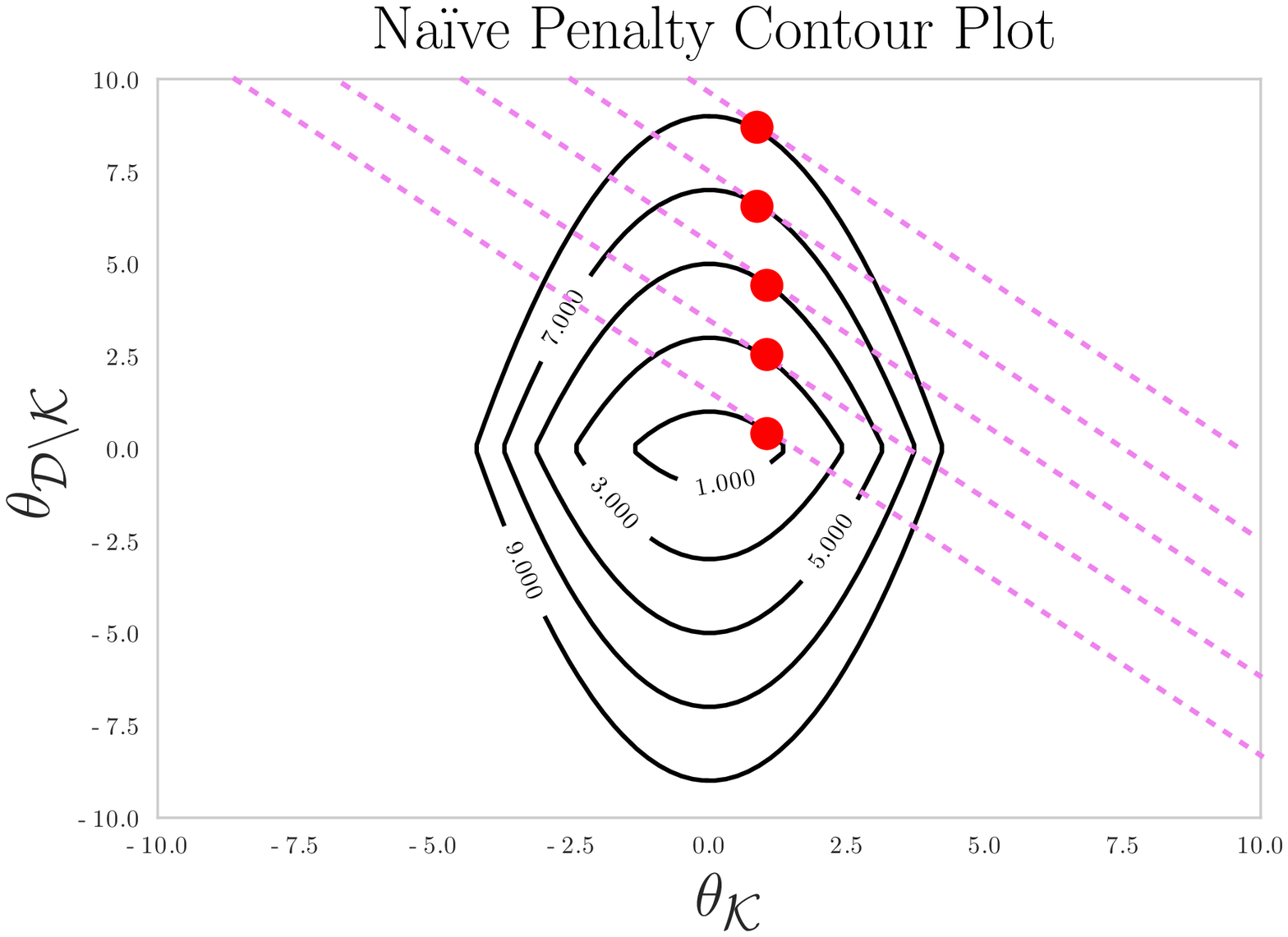}
  \caption{} \label{fig:sub1}
\end{subfigure}
\begin{subfigure}{0.33\textwidth}
  \centering
  \includegraphics[width=1\linewidth]{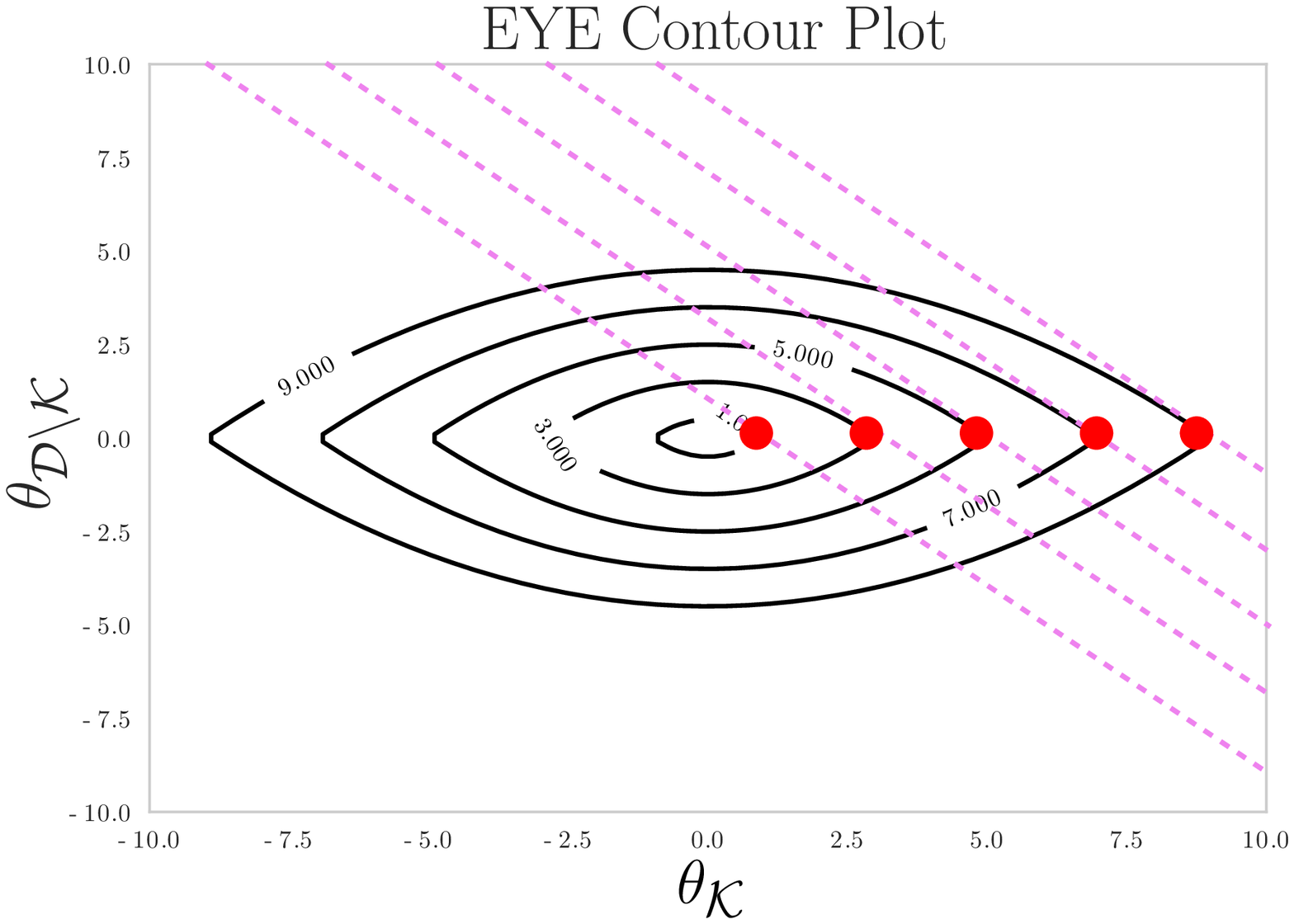}
  \caption{} \label{fig:sub2}
\end{subfigure}
\begin{subfigure}{0.33\textwidth}
  \centering
  \includegraphics[width=1\linewidth]{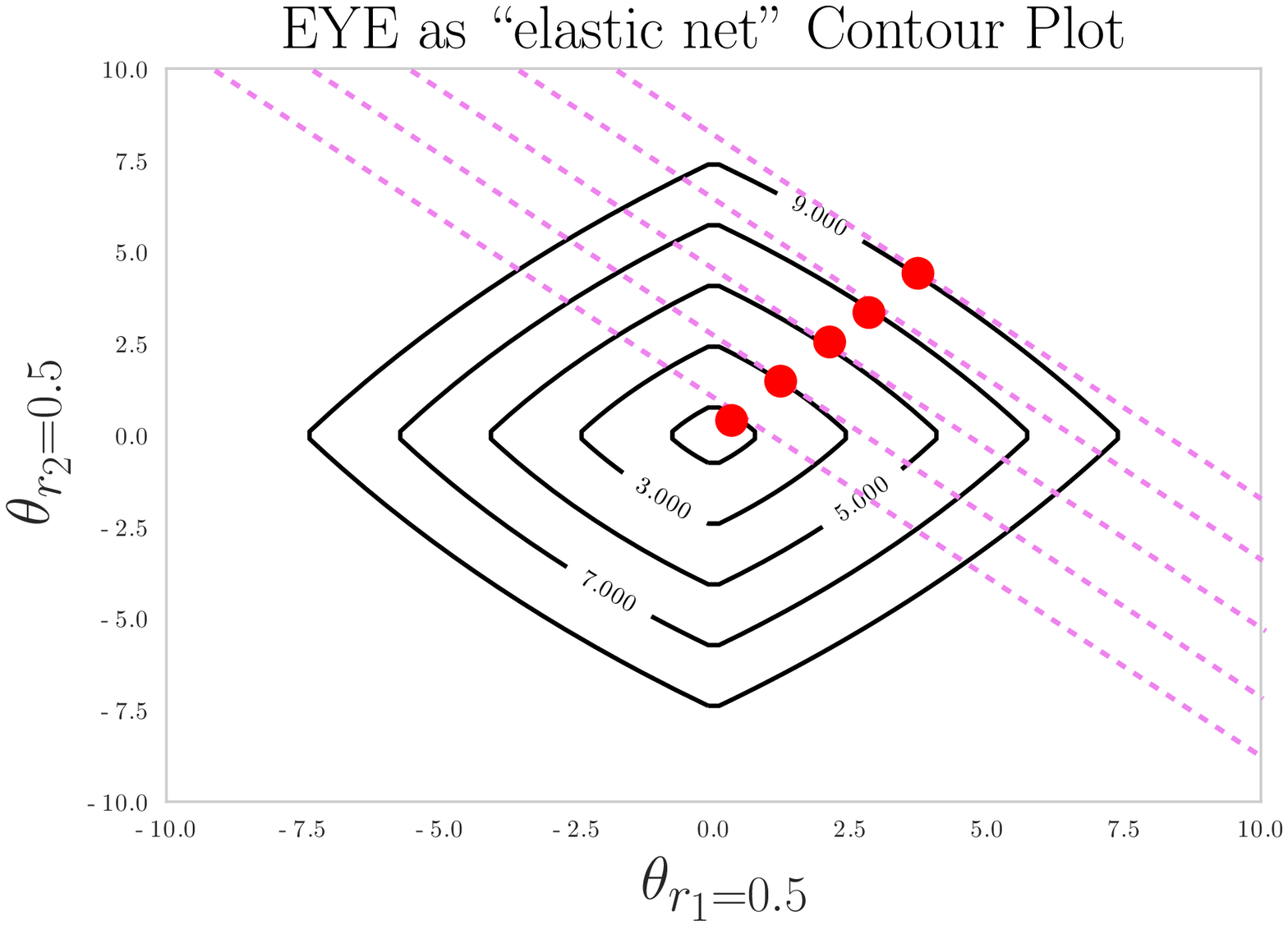}
  \caption{}\label{fig:sub3}
\end{subfigure}

\caption{\small Visualization of selected regularization penalties. Dashed violet lines denote level sets for the loss function when features are perfectly correlated; red dots are the optimal points for each feasible region. A large feasible region (level sets with large labeled values) corresponds to a small $\lambda$. \textbf{(a)} The naïve penalty ($\beta=0.5$) favors $\theta_{\mathcal{D}\setminus \mathcal{K}}$ as the feasible region grows. \textbf{(b)} EYE consistently favors $\theta_{\mathcal{K}}$. \textbf{(c)} When $\boldsymbol{r}=\boldsymbol{0.5}$, EYE produces a contour plot similar to elastic net. Setting $\boldsymbol{r}=\boldsymbol{0.5}$ represents a situation in which two features $i$ and $j$ are equally "known" and perfectly correlated. In this setting, $\hat \theta_i=\hat \theta_j$ (\ie, highly correlated known factors have similar weights)}
\label{fig:fig1}
\end{figure*}

\subsection{EYE Properties} \label{properties}

In this section, we give theoretical results for the proposed EYE penalty. We include detailed proofs in the Appendix$^1$. While the first three properties are general, the last three properties are valid in the least squares regression setting, \ie, {\small $Loss(\boldsymbol{\theta}, X, \boldsymbol{y}) = \frac{1}{2} \Vert \boldsymbol{y} - X \boldsymbol{\theta} \Vert_2^2$}. We focus on the least square regression setting because a closed form solution exists, though our method is applicable to the classification setting as well (demonstrated in section \ref{experiments}).

\textbf{EYE is a norm:} This comes for free as \textbf{Equation (\ref{eye-defn-orig})} is an atomic norm \cite{chandrasekaran2012convex}, thus, convex.

\textbf{EYE is $\beta$ free:} Similar to elastic net and the naïve penalty $q$,  EYE is a combination of the $l_1$ and $l_2$ norms, but it omits the extra parameter $\beta$. This leads to a quadratic reduction in the hyperparameter search space for EYE compared to elastic net and $q$. 

\textbf{EYE is a generalization of LASSO, $l_2$ norm, and ``elastic net'':} Setting $\boldsymbol{r}=\boldsymbol{1}$ and $\boldsymbol{0}$, we recover the $l_2$ norm and LASSO penalties, respectively. Relaxing $\boldsymbol{r}$ from a binary valued vector to a float valued vector, so that $\boldsymbol{r}=\boldsymbol{0.5}$, we get the elastic net shaped contour (\textbf{Figure \ref{fig:sub3}}). Elastic net is in quotes because the contour represents one particular level set, and elastic net is non-homogeneous.

\textbf{EYE promotes sparse models:}
Assuming $X^\top X = I$, the solution to EYE penalized least squares regression is sparse.
\textbf{Figure \ref{fig:sub4}} illustrates this effect in the context with other regularization penalties.

\textbf{EYE favors a solution that is sparse in {\small $\boldsymbol{\hat \theta}_{\mathcal{D}\setminus \mathcal{K}}$} and dense in {\small $\boldsymbol{\hat \theta}_{\mathcal{K}}$}:} \label{property2} In a setting in which covariates are perfectly correlated, {\small $\boldsymbol{\hat \theta}_{\mathcal{D}\setminus \mathcal{K}}$} will be set to exactly zero. Conversely, {\small $\boldsymbol{\hat \theta}_{\mathcal{K}}$} has nonzero entries. Moreover, the learned weights will be the same for every entry of {\small $\boldsymbol{\hat \theta}_{\mathcal{K}}$} (\textit{e.g.}, \textbf{Figure \ref{fig:sub3}}). This verifies the first part of the structure constraint. We also note that when the group of correlated features are all in {\small $\mathcal{D}\setminus \mathcal{K}$}, the objective function reverts back to LASSO, so that the weights are sparse, substantiating the second part of the structure constraint.

\textbf{EYE groups highly correlated known factors together:} 

If $\hat \theta_i \hat \theta_j > 0$ and the design matrix is standardized, then
\begin{equation*}
\resizebox{.47\textwidth}{!}{$
  \frac{|r_i^2 \hat \theta_i - r_j^2 \hat \theta_j|}{Z} \leq \frac{\sqrt{2 (1-\rho)} \Vert \boldsymbol{y} \Vert_2}{n\lambda}
  + |r_i-r_j| \left (1+\frac{\Vert (\boldsymbol{1}-\boldsymbol{r}) \odot \boldsymbol{\hat \theta} \Vert_1}{Z} \right )
$}
\end{equation*}

where $\rho$ is the sample covariance between $x_i$ and $x_j$, and {\small $Z = \sqrt{\Vert (\boldsymbol{1}-\boldsymbol{r}) \odot \boldsymbol{\hat \theta} \Vert_1^2 + 
    \Vert \boldsymbol{r} \odot \boldsymbol{\hat \theta} \Vert_2^2}$.}

This implies that when $r_i=r_j \neq 0$
 {\small $$\frac{|\hat \theta_i - \hat \theta_j|}{Z} \leq \frac{\sqrt{2(1-\rho)} \Vert \boldsymbol{y} \Vert_2}{r_i^2 n \lambda}$$}
\textit{I.e.}, the more correlated known important factors are, the more similar their weights will be. This is analogous to the grouping effect. 

\section{EXPERIMENTS} \label{experiments}

In this section, we empirically verify EYE's ability to yield credible models through a series of experiments. We compare EYE to a number of other regularization penalties across a range of settings using both synthetic and real data.

\subsection{Measuring Credibility} \label{measure_credibility}
\textbf{Criterion (i): density in the set of known relevant features and sparsity in the set of unknown}. In a two dimensional setting, we measure $\log |\frac{\theta_{\mathcal{K}}}{\theta_{\mathcal{D} \setminus \mathcal{K}}}|$ as a proxy for desirable weight structure (the higher the better). In a high-dimensional setting, highly correlated covariates form groups. For each group of correlated features, if known factors exist and are indeed important, then the shape of the learned weights should match $\boldsymbol{r}$ in the corresponding groups. \textit{E.g.}, given two correlated features $x_1$ and $x_2$ that are associated with the outcome, if $r_1=0$ and $r_2=1$, then $\theta_1=0$ and $\theta_2 \neq 0$. Thus, to measure credibility, we use the symmetric KL divergence, {\small $symKL(\boldsymbol{\hat\theta_g}', \boldsymbol{r}') = \frac{1}{2} \left ( KL(\boldsymbol{\hat\theta_g}' \Vert \boldsymbol{r}') + KL(\boldsymbol{r}' \Vert \boldsymbol{\hat\theta_g}') \right )$}, between the normalized absolute value of learned weights and the normalized $\boldsymbol{r}$ for each group $g$. For groups of relevant features that do not contain known factors, the learned weights should be sparse (\ie, all weight should be placed on a single feature within the group). Thus, we report {\small $\min_{\boldsymbol{x} \in \textit{one hot vectors}} symKL(\boldsymbol{x}, \boldsymbol{\hat \theta}')$} for such groups. As {\small $symKL$} decreases, the credibility of a model increases. Note that {\small $symKL$} only measures the shape of weights within \textit{each group} of correlated features and does not assume expert knowledge is correct (\eg, all weights within a group could be near zero). 

In our experiments on real data, we do not know the true underlying $\boldsymbol{\theta}$ and the partition of groups. In this case, we measure credibility by computing the fraction of known important factors in the top $n$ features sorted by the absolute feature weights learned by the model. We sweep $n$ from $1$ to $d$ and report the average precision (AP) between $|\boldsymbol{\hat \theta}|$ and $\boldsymbol{r}$.  

\textbf{Criterion (ii): maintained classification performance}. Recall that we want to learn a credible model without sacrificing model performance. That is, there should be no statistically significant difference in performance between a credible model and the best performing one (in this case, we focus on best linear models learned using other regularization techniques). We measure model performance in terms of the area under the receiver operating characteristic curve (AUC). In our experiments, we split our data into train, validation, and test sets. We train a model for each hyperparameter and bootstrap the validation set 100 times and record performance on each bootstrap sample. We want a model that is both accurate and sparse (measured using the Gini coefficient due to its desirable properties \cite{hurley2009comparing}). To ensure accuracy, for each regularization method, we remove models that are significantly worse than the best model in that regularization class using the validation set bootstrapped 100 times (p value set at .05). From this filtered set, we choose the sparsest model and report criteria (i) and (ii) on the held-out test set. 

\subsection{Experimental Setup and Benchmarks}
We compare EYE to the regularization penalties in \textbf{Table \ref{table3}} across various settings. We exclude ridge from our comparisons, because it produces a dense model (\textbf{Figure \ref{fig:sub4}}). In addition, we exclude adaptive LASSO because it requires an additional stage of processing. 

We set the weights, $\boldsymbol{w}$, in \textbf{Table \ref{table3}}, to mimic the effect of the $\boldsymbol{r}$. This gives a subset of the regularization techniques according to the same kind of expert knowledge that our proposed approach uses. In weighted LASSO and weighted ridge, the values in {\small $\boldsymbol{w}_{\mathcal{D}\setminus \mathcal{K}}$} were swept from $1$ to $3$ times the magnitude of the values in {\small $\boldsymbol{w}_{\mathcal{K}}$} to penalize unknown factors more heavily. For OWL, we set the weights in two ways. In the first case, we only penalize {\small $|\boldsymbol{\hat \theta}|_{[1]}$}, effectively recovering the {\small $l_{\infty}$} norm. In the second case, weights for the $m$ largest entries in {\small $\boldsymbol{\hat \theta}$} are set to be twice the magnitude of the rest, where $m$ is the number of known important factors. Note that a direct translation from known factors to weights is not possible in OWL, since the weights are determined based on the learned ordering. We implemented all models as a single layer perceptron with a softmax trained using the ADADELTA algorithm \cite{zeiler2012adadelta} minimizing the logistic loss. 

\subsection{Validation on Synthetic Datasets}
To test EYE under a range of settings, we construct several synthetic datasets \footnote[2]{code available at \url{https://github.com/nathanwang000/credible_learning}}. In all experiments, we generate the data and run logistic regression with EYE and each regularization benchmark. In all of our experiments on synthetic data, we found no statistically significant differences in AUC, thus satisfying the performance constraint. These experiments expose  the limitations of the na\"ive penalty, measure sensitivity to noise and to correlation in covariates, explore different shapes of $\boldsymbol{r}$, and examine the effect of the accuracy of expert knowledge on credibility. In all cases, the EYE penalty leads to the most credible model, validating our theoretical results. 

\subsubsection{Limitations of the Naïve Penalty: Sensitivity to Hyperparameters}

 The naïve penalty $q$ appears to be a natural solution for building credible linear models. However, since $q$ is non-homogeneous, as the constraint region grows, the models begin to prefer features \textit{not} in $\mathcal{K}$. Since small $\lambda$ corresponds to a large constraint region, we vary $\lambda$ to expose this undesirable behavior. 

 We sample $100$ data points uniformly at random from $-2.5$ to $1.5$ to create $\boldsymbol{v}$. We set $X = [\boldsymbol{v}, \boldsymbol{v}]$ to produce two perfectly correlated features with one known factor. We set $\boldsymbol{\theta}=[1,1]$ (note that since the two features are perfectly correlated, it doesn't matter how $\boldsymbol{\theta}$ is assigned), and assign the label $\boldsymbol{y}$ as $\mathds{1}_{\boldsymbol{\theta}^\top \boldsymbol{x} > 0}(\boldsymbol{x})$ for each data point $\boldsymbol{x}$. 

 \textbf{Figure \ref{fig:fig2}} shows the log ratio for credibility for different settings of $\lambda$ and $\beta$. First note that as $\lambda$ approaches zero, the log ratio approaches $0$ for all methods because the models are effectively unconstrained. With nontrivial $\lambda$ and large $\beta$, both EYE and the naïve penalty result in high credibility. This is expected as a large $\beta$ will constrain known important factors less, thus placing more weight on them. 
 For $\beta$ in the lower range, the log ratio is negative because the naïve penalty penalizes known features more. For $\beta$ in the middle range, the log ratio varies from credible to non-credible, exhibiting the artifact of non-homogeneity (the penalty contour is elongated along $\theta_{\mathcal{K}}$ as $\lambda$ decreases, thus again favoring $X_{\mathcal{D} \setminus \mathcal{K}}$). Since we want the log ratio$>0$ for all nontrivial $\lambda$, the naïve penalty with $\beta < 0.8$ fails. 

\begin{figure}[h!]
\centering
\begin{subfigure}{0.48\textwidth}
  \centering
  \includegraphics[width=1\linewidth]{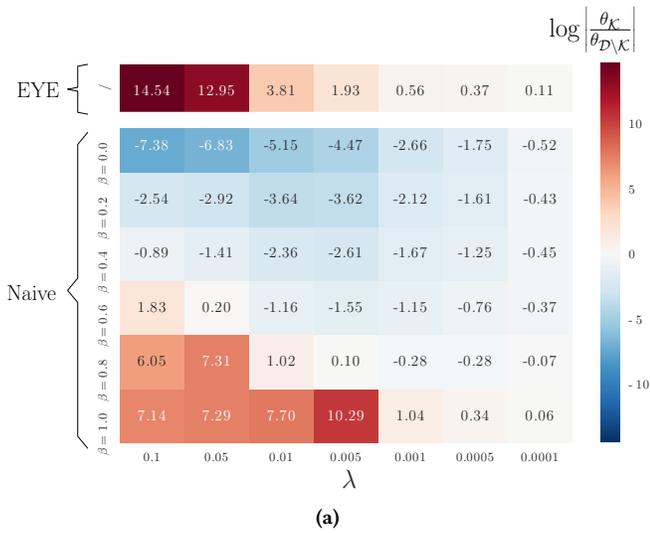}
  \caption{}
  \label{fig:fig2}
\end{subfigure}%
\quad
\begin{subfigure}{0.48\textwidth}
  \centering
  \includegraphics[width=1\linewidth]{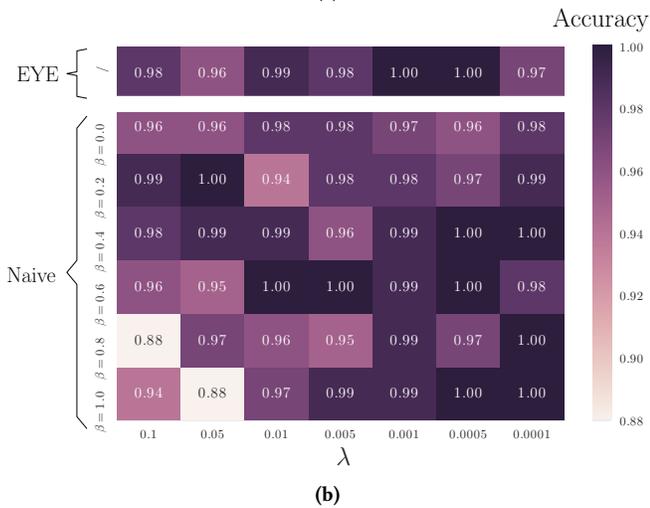}
  \caption{}
  \label{fig:fig7}
\end{subfigure}
\caption{A comparison of the naïve penalty and EYE. \textbf{(a)} EYE meets the structural constraint better than naïve penalty with small and mid-ranged $\beta$ \textbf{(b)} EYE has better performance than naïve Penalty with large $\beta$.}
\end{figure}

The naïve penalty with large $\beta$ also fails to produce credible models because the resulting models have worse classification performance. In particular, when $\beta>0.8$, the naïve penalty overemphasizes the relevancy of known important factors. 
 As shown in \textbf{Figure \ref{fig:fig7}}, the naïve penalty with large $\beta$ performs considerably worse in terms of accuracy than EYE for large $\lambda$. On small $\lambda$, their performance are comparable. This is expected because EYE introduces less bias towards known important factors.


\subsubsection{Varying the Degree of Collinearity}

We can show theoretically that EYE results in a credible model when features are highly correlated. However, the robustness of EYE in the presence of noise is unknown. To explore how EYE responds to changes in correlation between features, we conduct an experiment in a high-dimensional setting. 


\begin{figure}
\centering
  \includegraphics[width=0.8\linewidth]{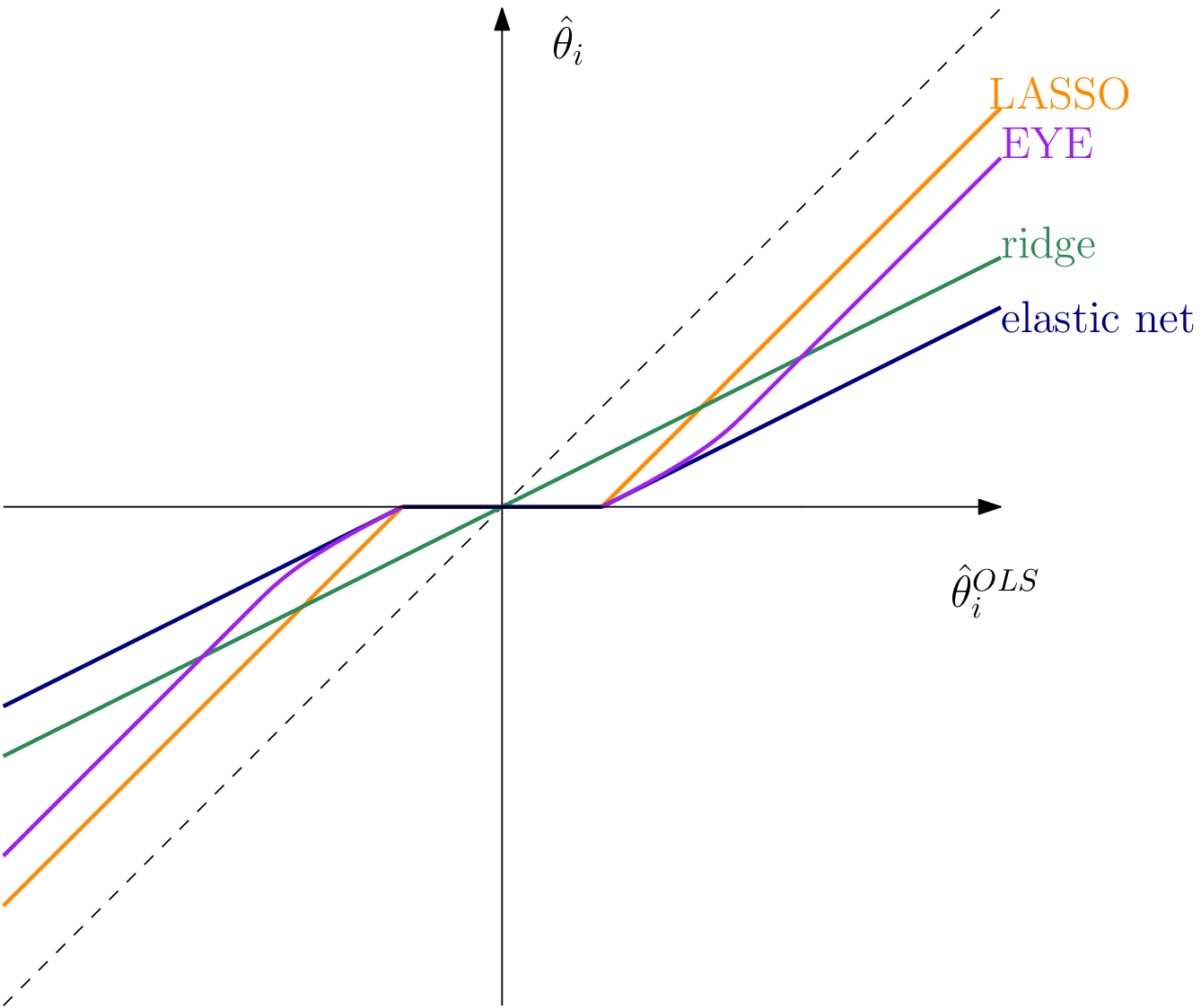}  
  \caption{When the design matrix is orthonormal, EYE, elastic net, and LASSO will set features with small ordinary least squares solution to exactly $0$. In contrast, ridge is dense.} \label{fig:sub4}

\begin{subfigure}{0.4\textwidth}
  \centering
  \includegraphics[width=1\linewidth]{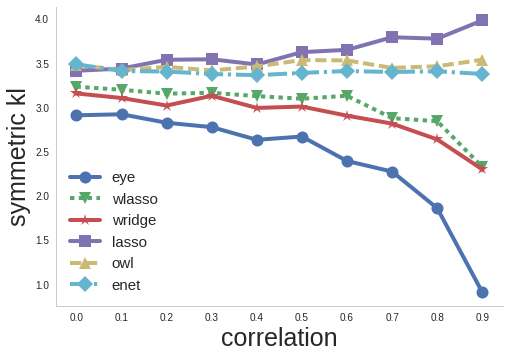}
	\caption{}
  \label{fig:fig4}
\end{subfigure}
\begin{subfigure}{0.4\textwidth}
  \centering
  \includegraphics[width=1\linewidth]{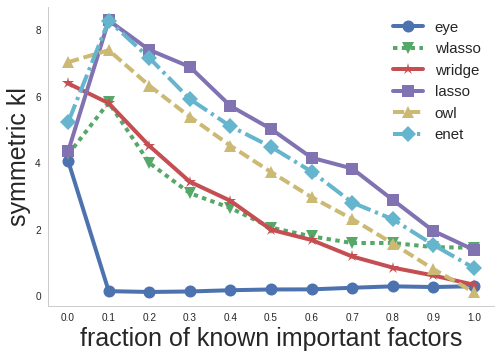}
	\caption{}
  \label{fig:fig5}
\end{subfigure}
\caption{ Comparisons of EYE with other methods under various settings \textbf{(a)} EYE leads to the most credible models in all correlations. \textbf{(b)} EYE leads to the most credible model for all shapes of $\boldsymbol{r}$.}
\end{figure}


We generate $10$ groups of data, each having $30$ features, with $15$ in $\mathcal{K}$. We assigned each group a correlation score from $0$ to $0.9$ (here, we exclude the perfectly correlated case as it will be examined in detail in the next experiment). Intra-group feature correlations are fixed to the group's correlation score, while inter-group feature correlations are $0$. 

\textbf{Figure \ref{fig:fig4}} plots the {\small $symKL$} for each group. Moving from left to right, the correlation increases in step size of $0.1$ from $0$ to $0.9$. As correlation increases, the EYE regularized model achieves the smallest {\small $symKL$}, and becomes the most credible model. In comparison, the other approaches do not achieve the same degree of credibility though, weighted LASSO and weighted ridge do exhibit a similar trend. However, since weighted LASSO fails to capture denseness in known important factors and weighted ridge fails to capture sparseness in unknown features, EYE leads to a more credible model. As correlation increases, LASSO actually produces a less credible model (as expected). 

\subsubsection{Varying Percentage of Known Important Factors}
Besides varying correlation, we also vary the percentage of known important factors within a group of correlated features. We observe that EYE is consistently better than other methods.

In this experiment, we generate groups of data $\mathcal{C}_i$ where $i=0,...,10$, each having $10$ features. Features in each group are perfectly correlated, and features across groups are independent. Each group has a different number of features in {\small $\mathcal{K}$}, \eg, group $0$ has $0$ known relevant factors and group $10$ has $10$ known important factors. 

\textbf{Figure \ref{fig:fig5}} plots the {\small $symKL$} for each group of features. The groups are sorted by {\small $|\mathcal{C}_i\cap\mathcal{K}|$}. When {\small $|\mathcal{C}_i\cap\mathcal{K}|=0$}, the model should be sparse. Indeed, for group $0$, we observe that EYE, LASSO, and weighted LASSO do equally well (EYE in fact degenerates to LASSO in this case), closely followed by elastic net. Weighted ridge and OWL, on the other hand, do poorly since they encourage dense models. For other groups, EYE penalty achieves the best result (lowest {\small $symKL$}). This can be explained by property \ref{property2} as EYE sets the weights the same for correlated features in {\small $\mathcal{K}$} while zeroing out weights in {\small $\mathcal{D} \setminus \mathcal{K}$}. Again, LASSO performed the worst overall because it ignores $\boldsymbol{r}$ and is sparse even when $\boldsymbol{r}$ is dense.


\subsubsection{Varying Accuracy of Expert Knowledge}

The experiments above only test cases where $\boldsymbol{\theta}$ is elementwise positive and where expert knowledge is correct (\ie, the features identified by the expert were indeed relevant). To simulate a more general scenario in which the expert may be wrong, we use the following generative process:

{\small
\begin{enumerate}[leftmargin=*,topsep=0.5pt]
\itemsep0em 
\item Select the number of independent groups, $n$ $\sim$ Poisson($10$)
\item For each group $i$ in $n$ groups
\begin{enumerate}[nolistsep]
\item Sample a group weight, $w^{(i)}$ $\sim$ Normal($0$,$1$)
\item Sample the number of features, $m^{(i)}$ $\sim$ Poisson($20$)
\item Sample known important factor indicator array, $\boldsymbol{r}^{(i)}$ $\sim$ Bernoulli($0.5$)$^{m^{(i)}}$
\item Assign true relevance $\boldsymbol{\theta}^{(i)} \in \mathbb{R}^{m^{(i)}}$ by distributing $w^{(i)}$ according to $\boldsymbol{r}^{(i)}$ (\eg, if $w^{(i)}=3$ and $\boldsymbol{r}^{(i)}=[0,1,1]$, then $\boldsymbol{\theta}^{(i)}=[0,1.5,1.5]$)
\end{enumerate}
\item Generate covariance matrix $C$ such that intra-group feature correlation=$0.95$ and inter-group feature correlation=$0$
\item Generate $5000$ i.i.d. samples $\boldsymbol{x}_i\in \mathbb{R}^{\sum_{i=1}^{n} m^{(i)}}$ $\sim$ Normal($\boldsymbol{0}, C$)
\item Choose label $y_i$ $\sim$ Bernoulli{\small $(sigmoid (\boldsymbol{\theta}^\top \boldsymbol{x}_i))$} where $\boldsymbol{\theta}$ is the concatenated array from $\boldsymbol{\theta}^{(i)}$
\end{enumerate}
}

Generating data this way covers cases where expert knowledge is wrong as feature group relevance and $\boldsymbol{r}$ are independently assigned. It also allows the number of features and weights for each group to be different. \textbf{Table \ref{table1}} summarizes performance and credibility for each method averaged across 100 runs. EYE achieves the lowest sum of {\small $symKL$} for each group of correlated features. In terms of AUC, the best models for each penalty are comparable, confirming that EYE is able to recover from the expert's mistakes.

\begin{table}[t]
        \centering
        \caption{EYE leads to the most credible model on a synthetic dataset (mean $\pm$ stdev)} \label{table1}
        \begin{tabular}{lcc}
            \toprule
            Method     & $\sum_{g=1}^{n} symKL_g$ & AUC\\ 
            \midrule
            EYE 	& \textbf{0.442} $\pm$ 0.128 &     $0.900\pm0.044$ \\
	        wLASSO &	$0.929 \pm 0.147$ & $0.898\pm0.044$\\
	        wridge 	& $1.441 \pm 0.241$ & $0.899 \pm 0.045$ \\
    	    LASSO & $2.483 \pm 0.440$ & $0.898\pm 0.044$ \\
    	    elastic net & $2.673 \pm 0.399$ & $0.893\pm 0.044$ \\
    	    OWL & $3.125 \pm 0.329$ & $0.900\pm0.044$ \\
            \bottomrule
        \end{tabular}
\end{table}

\begin{table*}[t]
   
\centering
    \begin{threeparttable}
        \centering
        \caption{EYE leads to the most credible model on both the \textit{C. diff} and \textit{PhysioNet Challenge} datasets; it keeps more of the factors identified in the clinical literature, while performing on par with other regularization techniques; it also has very sparse weights, second only to the model that just uses features in the risk factors}\label{table2}
        \begin{tabular}{lccccccc}
            \toprule
             & \multicolumn{3}{c}{\textit{C. diff}} & & \multicolumn{3}{c}{\textit{PhysioNet Challenge}}\\
                         \midrule
            Method     & AP & AUC & sparsity$^+$  & \ \ \ \ \ \ \ \ \ \  \ \ \ \  \ \ \ & AP & AUC & sparsity$^+$ \\
            \midrule
            expert-features-only \ \ \ \  \ \ \ \  \ \ \ \  \ \ \ \ \ \ \ \ \ & $1^*$ & 0.598 & \textbf{0.998} & &$1^*$ & 0.754 & \textbf{0.877} \\
            EYE	    & \textbf{0.204} &  0.753 & 0.980&     & \textbf{0.671} &  0.815 & 0.794\\
            wLASSO	& 0.033 & 0.764 & 0.884& & 0.300 & 0.810 & 0.824 \\
            LASSO	& 0.032 & 0.760 & 0.856& & 0.131 & \textbf{0.823} & 0.779\\
            wridge	& 0.031 & \textbf{0.768} & 0.755& & 0.209 & 0.810 & 0.069\\
            elastic net	& 0.031 & 0.754 & 0.880& & 0.153 & 0.818 & 0.649\\
    	    EYE-random-r & 0.031 & 0.748 & 0.936& & 0.589 & 0.792 & 0.779\\
    	    OWL	    & 0.028 & 0.548 & 0.544&     & 0.108 & 0.794 & 0.046\\
            \bottomrule
       \end{tabular}
        \begin{tablenotes}
            \small 
            \item $^+$ percentage of near-zero feature weights, where near-zero is defined as $<0.01$ of the largest absolute feature weight
            \item * expert-features-only logistic regression trivially achieves AP of $1$ simply because it only uses expert features
        \end{tablenotes}
                   \end{threeparttable}
\end{table*}

\subsection{Application to a Real Clinical Prediction Task} \label{real_experiment}

After verifying desirable properties in synthetic datasets, we apply EYE to a large-scale clinical classification task. In particular, we consider the task of identifying patients at greatest risk of acquiring an infection during their hospital stay.  We selected a task from healthcare since credibility and interpretability are critical to ensuring the safe adoption of such models. We focus on predicting which patients will acquire a \textit{Clostridium difficile} infection (CDI), a particularly nasty healthcare-associated infection. Using electronic health record (EHR) data from a large academic US hospital, we aim to learn a credible model that produces accurate \textit{daily} estimates of patient risk for CDI. 

\subsubsection{The Dataset.}
We consider all adult hospitalizations between 2010 and 2015. We exclude hospitalizations in which the patient is discharged or diagnosed with CDI before the 3rd calendar day, since we are interested in healthcare-acquired infections (as opposed to community-acquired). Our final study population consists of $143,602$ adult hospitalizations. Cases of CDI are clinically diagnosed by positive laboratory test. We label a hospitalization with a positive laboratory test for CDI as +1, and 0 otherwise. $1.09\%$ of the study population is labeled positive.

\subsubsection{The Task.}
We frame the problem as a prediction task: the goal is to predict whether or not the patient will be clinically diagnosed with CDI at some point in the future during their visit. In lieu of a single prediction at 24 hours, we make predictions every 24 hours. To generate a single AUC given multiple predictions per patient, we classify patients as high-risk if their risk ever exceeds the decision threshold, and low-risk otherwise. By sweeping the decision threshold, we generate a single receiver operating characteristic curve and a single AUC in which each hospitalization is represented exactly once.

\subsubsection{Feature Extraction.}
We use the same feature extraction pipeline as described in \cite{Oh_2018}. In particular, we extract high-dimensional feature vectors for each day of a patient's admission from the structured contents of the EHR (\textit{e.g.}, medication, procedures, in-hospital locations etc.). Most variables are categorical and are mapped to binary features. Continuous features are either binned by quintiles or well-established reference ranges (\textit{e.g.}, a normal heart rate is 60-100 beats per minute). If a feature is not measured (\textit{e.g.}, missing vital), then we explicitly encode this missingness. Finally, we discard rare features that are not present in more than .05\% of the observations. This feature processing resulted in 4,739 binary variables. Of these variables, 264 corresponded to known risk factors. We identified these variables working with experts in infectious disease who identified key factors based on the literature \cite{garey2008clinical,dubberke2011development,wiens2014learning}.

\subsubsection{Analysis.}
We train and validate the models on data from the first five years (n=$444,184$ days), and test on the held-out most recent year (n=$217,793$ days). Using the training data, we select hyperparameters using a grid search for $\lambda$ and $\beta$ from $10^{-10}$ to $10^{10}$ and $0$ to $1$ respectively. The final hyperparameters are selected based on model performance and sparsity as detailed in section \ref{measure_credibility}. 

For each regularization method, we report the AUC on the held-out test set, and the average precision (AP) between {\small $|\boldsymbol{\hat \theta}|$} and $\boldsymbol{r}$ (see Section \ref{measure_credibility}). \textbf{Table \ref{table2}} summarizes the results on the test set with various regularizations. 

Relative to the other common regularization techniques, EYE achieves an AP that is an order of magnitude higher, while maintaining good predictive performance. Moreover, EYE leads to one of the sparsest models, increasing model interpretability.

For comparison, we include a model based on only the 264 expert features (trained using $l_2$ regularized logistic regression) ``expert-features-only.'' This baseline trivially achieves AP of $1$, since it only uses expert features, but performs poorly relative to the other tasks. This confirms that simply retaining expert features is not enough to solve this task.

In addition, we include a baseline, "EYE-random-r", in which we randomly permuted $\boldsymbol{r}$. This corresponds to the setting where the expert is incorrect and is providing information about features that may be irrelevant. In this setting, EYE achieves a high AUC and low AP. This confirms that EYE is not severely biased by incorrect expert knowledge. Moreover, we believe this to be a feature of the approach, since it can highlight settings in which the data and expert disagree.


\subsection{Application to PhysioNet Challenge Dataset}

To further validate our approach, we turn to a publicly available benchmark dataset from PhysioNet \cite{PhysioNet}. In this task, the goal is to predict in-hospital mortality using EHR data collected in intensive care units (ICUs). Similar to above using the EYE penalty we trained a model and evaluated it in terms of predictive performance, average precision (AP), and model sparsity. 

\subsubsection{The Dataset.} We use the ICU data provided in the PhysioNet Challenge 2012 \cite{silva2012predicting} to train our model. This challenge utilizes a subset of the MIMIC-III dataset. We focus on this subset rather than using the entire dataset, since the goal is not to achieve state-of-the-art in in-hospital mortality prediction, but simply to evaluate the performance of the EYE penalty. The challenge data consist of three sets, each set containing data for 4000 patients. In our experiments, we use set A, since it is the only publicly labeled subset. We split the data randomly, reserving $25\%$ as the held-out test set.

\subsubsection{The Task.}  Using data collected during the first two days of an ICU stay, we aim to predict which patients survive their hospitalizations, and which patients do not. In contrast to the \textit{C. diff} task, here, we make a single prediction per patient at 48 hours. 

\subsubsection{Feature Extraction.} The PhysioNet challenge dataset has considerably fewer features relative to the earlier task. In total, for each patient the data contain four general descriptors (\eg, age) and 37 time-varying variables (\eg, glucose, pH, etc.) measured possibly multiple times during the first 48 hours of the patient's ICU stay. We describe our feature extraction process below. Since again the goal was not state-of-the-art prediction on this particular task, we performed standard preprocessing without iteration/optimization.

We represent each patient by a vector containing 130 features. More specifically, for each time-varying variable we compute the maximum, mean, and minimum over the 48 hour window, yielding 111 features. In addition, for each of the 15 time-varying variables used in the Simplified Acute Physiology Score (SAPS-I) \cite{le1984simplified} we extract the most abnormal value observed within the first 24 hours,based on the SAPS scoring system. We concatenate these 126 features along with the 4 general descriptors producing a final vector of length 130. 
Out of the 130 variables, we consider the 15 SAPS-I variables along with age as expert knowledge. SAPS-I is a scoring system used to predict ICU mortality in patients greater than the age of 15 and thus corresponds to factors believed to increase patient risk.

\subsubsection{Analysis.}

Using the training data, we select hyperparameters in the same way we did earlier. As with the previous experiment on the \textit{C. diff} dataset, for each regularization method, we report both AUC and AP on the held-out test set for this task. Again, we compared the model learned using the EYE penalty to the other baselines. \textbf{Table \ref{table2}} summarizes our results on the held-out test set.

Overall, we observed a similar trend as to what we observed for the \textit{C. diff} dataset. Compared to the other common regularization techniques, EYE achieves significantly higher AP and results in a sparse model. In terms of discriminative performance it performs on par with the other techniques. Again, we see that a model based on the expert features alone (i.e., \textit{expert-features-only}) performs worse than the other regularization techniques. However, the difference in performance is not as striking as it was earlier. This suggests that perhaps the additional features (beyond the 16 SAPS-I features) do not provide much complementary information. Interestingly, the model using randomly permuted $\boldsymbol{r}$ ("EYE-random-r")  achieves high AUC and AP. We suspect this may be due to the amount of collinearity present in the data. The non-expert and expert features are highly correlated with one another and thus both subsets are predictive (\textit{i.e.}, supported by the data). 

\section{Discussion \& Conclusion} \label{discussion}
In this work, we extended the notion of interpretability to credibility and presented a formal definition of credibility in a linear setting. We proposed a regularization penalty, EYE, that encourages such credibility. Our proposed approach incorporates domain knowledge about which factors are known (or believed) to be important. Our incorporation of expert knowledge results in increased credibility, encouraging model adoption, while maintaining model performance. Through a series of experiments on synthetic data, we showed that sparsity inducing regularization such as LASSO, weighted LASSO, elastic net, and OWL do not always produce credible models. In contrast, EYE produces a model that is provably credible in the least squares regression setting, and one that is consistently credible across a variety of settings. 

Applied to two large-scale patient risk stratification tasks, EYE produced a model that was significantly better at highlighting known important factors, while being comparable in terms of predictive performance with other regularization techniques. Moreover, we demonstrated how the proposed approach does not lead to worse performance when the expert is wrong. This is especially important in a clinical setting, where some relationships between variables and the outcome of interest may be less well-established. 

There are several important limitations of the proposed approach. We focused on a linear setting and one form of expert knowledge. In the future, we plan to extend the notion of credibility to other settings. Furthermore, we do not claim that EYE is the optimal approach to yield credibility (we give no proof on that). Compared to other regularization penalties considered in this paper, EYE introduces the least amount of bias, while striving to attain credibility. 

While interpretable models have garnered attention in recent years, increased interpretability should not have to come at the expense of decreased credibility. Predictive performance and sparsity being equal, a data-driven model that reflects what is known or well-accepted in one’s domain (in addition to what is unknown, but reflected in the data) is preferred over a purely data-driven model that highlights unusual features due to collinearity in the data. Moreover, correlations can be fragile and break over time; thus, credible models that select those features that are known to be associated with the outcome of interest may also be more robust to such changes over time. 

Finally, though we focused on credibility, our proposed regularization technique could be extended to other settings in which the user would like to guide variable selection. For example, instead of encoding knowledge pertaining to which variables are known risk factors, $\mathbf{r}$ could encode information about which variables are actionable. This in turn could lead to more \textit{actionable} models.

\section{Acknowledgement}

This work was supported by the National Science Foundation (NSF award no. IIS-1553146); the National Institute of Allergy and Infectious Diseases of the National Institutes of Health (grant no. U01AI124255). The views and conclusions in this document are those of the authors and should not be interpreted as necessarily representing the official policies, either expressed or implied, of the National Science Foundation nor the National Institute of Allergy and Infectious Diseases of the National Institutes of Health.

\bibliographystyle{ACM-Reference-Format}
\balance
\bibliography{ref.bib}

\newpage
\section{Appendix}

This Appendix includes details of the proofs for properties in \ref{properties}. We assume $\lambda > 0$ because otherwise the model is not regularized.

\subsection{Derivation of original EYE penalty}

First note that $\left \{\boldsymbol{x} \mid q(\boldsymbol{x})=c \right \}$ is the convex contour plot of $q$ for $c \in \mathbb{R}$. We set $c$ so that the slope in the first quadrant between known important factor and unknown feature is $-1$. 

Since we only care about the interaction between known and unknown risk factors and that the contour is symmetric about the origin, WLOG, let y be the feature of unknown importance and x be the known important factor and $y \geq 0$, $x \geq 0$.


 \begin{alignat}{2}
 &2 \beta y + (1-\beta) x^2 = c \notag\\
 &\Rightarrow \quad y = \frac{c}{2\beta} - \frac{(1-\beta) x^2}{2 \beta} \notag\\
 &\Rightarrow \quad y = 0 \Rightarrow x = \sqrt{\frac{c}{1-\beta}} \notag\\ 
 &\Rightarrow \quad f'(x) = -\frac{(1-\beta)}{\beta}x \notag\\
 &\Rightarrow \quad f'(\sqrt{\frac{c}{1-\beta}}) = -\frac{1-\beta}{\beta} \sqrt{\frac{c}{(1-\beta)}} = -1 \notag\\
 &\Rightarrow \quad c = \frac{\beta^2}{1-\beta} \notag\\
 &\Rightarrow \quad 2 \beta y + (1-\beta) x^2 = \frac{\beta^2}{1-\beta}
\end{alignat}

Thus, we just need $q(\boldsymbol{x}) = \frac{\beta^2}{1-\beta}$. The rest deals with scaling of the level curve. We define EYE penalty as a an atomic norm $ \Vert \cdot \Vert_A$ introduced in \cite{chandrasekaran2012convex}: $ \Vert \boldsymbol{x} \Vert_A := \inf \left \{t>0 \mid \boldsymbol{x} \in t conv(A) \right \}$ where $conv$ is the convex hull operator of its argument set $A$.

Let $A= \left \{\boldsymbol{x} \mid q(\boldsymbol{x}) \leq \frac{\beta^2}{1-\beta} \right \}$. Using the fact that the sublevel set of q is convex, we have
\begin{equation}
eye(\boldsymbol{x}) = \inf \left \{t>0 \mid \boldsymbol{x} \in \left \{t \boldsymbol{x} \mid q(\boldsymbol{x}) \leq \frac{\beta^2}{1-\beta} \right \} \right \}
\end{equation}

\subsection{EYE has no extra parameter} \label{beta_free}

To show $\beta$ is unused in EYE, we show that $\beta$ conserves the shape of the contour, because the scaling of EYE can be absorbed in to $\lambda$. 

\begin{proof}

 Consider the contour $B_1 = \left \{\boldsymbol{x}: eye_{\beta_1}(\boldsymbol{x}) = t \right \}$ and $B_2 = \left \{\boldsymbol{x}: eye_{\beta_2}(\boldsymbol{x}) = t \right \}$

  We want to show $B_1$ is similar to $B_2$

  case1: $t = 0$, then $B_1$ = $B_2 = \{0\}$ because EYE is a norm.
  
  case2: $t \neq 0$

  we can equivalently write $B_1$ and $B_2$ as

  $B_1 = t \left \{\boldsymbol{x}: \boldsymbol{x} \in \left \{\boldsymbol{x} \mid q_{\beta_1}(\boldsymbol{x}) = \frac{\beta_1^2}{1-\beta_1} \right \} \right \}$

  $B_2 = t \left \{\boldsymbol{x}: \boldsymbol{x} \in \left \{\boldsymbol{x} \mid q_{\beta_2}(\boldsymbol{x}) = \frac{\beta_2^2}{1-\beta_2} \right \} \right \}$

  let $B_1' = \left \{\boldsymbol{x}: \boldsymbol{x} \in \left \{\boldsymbol{x} \mid q_{\beta_1}(\boldsymbol{x}) = \frac{\beta_1^2}{1-\beta_1} \right \} \right \}$
  and 
  
  $B_2' = \left \{\boldsymbol{x}: \boldsymbol{x} \in t \left \{\boldsymbol{x} \mid q_{\beta_2}(\boldsymbol{x}) = \frac{\beta_2^2}{1-\beta_2} \right \} \right \}$

  \begin{claim}
  $B_2' = \frac{\beta_2 (1-\beta_1)}{\beta_1 (1-\beta_2)} B_1'$
  \end{claim} 

  It should be clear that if this claim is true then $B_1$ is similar to $B_2$ and we are done

  Take $\boldsymbol{x} \in B_1'$, then $q_{\beta_1}(\boldsymbol{x}) = 2 \beta_1  \Vert (\boldsymbol{1}-\boldsymbol{r}) \odot \boldsymbol{x} \Vert_1 + (1-\beta_1)  \Vert \boldsymbol{r} \odot \boldsymbol{x} \Vert_2^2 = \frac{\beta_1^2}{1-\beta_1}$

  let $\boldsymbol{x'} = \frac{\beta_2 (1-\beta_1)}{\beta_1 (1-\beta_2)} \boldsymbol{x}$

{\small
  \begin{align*}
  q_{\beta_2}(\boldsymbol{x'}) &= 2 \beta_2  \Vert (\boldsymbol{1}-\boldsymbol{r}) \odot \boldsymbol{x'} \Vert_1 +
  (1-\beta_2)  \Vert \boldsymbol{r} \odot \boldsymbol{x'} \Vert_2^2\\
  &= \frac{2 \beta_2^2 (1-\beta_1)}{\beta_1 (1-\beta_2)}  \Vert (\boldsymbol{1}-\boldsymbol{r}) \odot \boldsymbol{x} \Vert_1 + 
  \frac{\beta_2^2 (1-\beta_1)^2}{\beta_1^2 (1-\beta_2)}  \Vert \boldsymbol{r} \odot \boldsymbol{x} \Vert_2^2\\
  &= \frac{\beta_2^2 (1-\beta_1)}{\beta_1^2 (1-\beta_2)} (2 \beta_1  \Vert (\boldsymbol{1}-\boldsymbol{r}) \odot \boldsymbol{x} \Vert_1 +
  (1-\beta_1)  \Vert \boldsymbol{r} \odot \boldsymbol{x} \Vert_2^2)\\
  &= \frac{\beta_2^2 (1-\beta_1)}{\beta_1^2 (1-\beta_2)} \frac{\beta_1^2}{1-\beta_1} \\
  &= \frac{\beta_2^2}{1-\beta_2}
  \end{align*}
}
  so $\boldsymbol{x'} \in B_2'$. Thus $\frac{\beta_2 (1-\beta_1)}{\beta_1 (1-\beta_2)} B_1' \subset B_2'$. The other direction is similarly proven.

\end{proof}

\subsection{Equivalence with the triangular form of EYE penalty}

In this section, we prove \textbf{Equation (\ref{eye-defn-orig})} and \textbf{(\ref{eye-defn})} are equivalent.

\begin{proof} 
     
Since $\beta$ can be arbitrarily set (\ref{beta_free}), fix $\beta$=0.5, then \textbf{Equation (\ref{eye-defn-orig})} becomes

\begin{equation}
\resizebox{.47\textwidth}{!}{$
     eye(\boldsymbol{x}) = \inf \left \{t>0 \mid \boldsymbol{x} \in t \left \{ \boldsymbol{x} \mid 
     2  \Vert (\boldsymbol{1}-\boldsymbol{r}) \odot \boldsymbol{x} \Vert_1 +  \Vert \boldsymbol{r} \odot \boldsymbol{x} \Vert_2^2 = 1 \right \} \right \}
$}
\end{equation}

Assume $\boldsymbol{x} \neq 0$ and denote

$eye(\boldsymbol{x}) := t$, then $\boldsymbol{x} \in t \left \{ \boldsymbol{x} \mid  2  \Vert (\boldsymbol{1}-\boldsymbol{r}) \odot \boldsymbol{x} \Vert_1 +  \Vert \boldsymbol{r} \odot \boldsymbol{x} \Vert_2^2 = 1 \right \}$, that is $\frac{2 \Vert (\boldsymbol{1}-\boldsymbol{r}) \odot \boldsymbol{x} \Vert_1}{t} + \frac{ \Vert \boldsymbol{r} \odot \boldsymbol{x} \Vert_2^2}{t^2} = 1$
     
As this is a quadratic equation in t and from assumption we know $t>0$ (EYE being a norm and $\boldsymbol{x} \neq 0$), solving for $t$ yields:     

\begin{equation} \label{tmp-derivation}
     t = \Vert (\boldsymbol{1}-\boldsymbol{r}) \odot \boldsymbol{x} \Vert_1 + 
     \sqrt{ \Vert (\boldsymbol{1}-\boldsymbol{r}) \odot \boldsymbol{x} \Vert_1^2 +  \Vert \boldsymbol{r} \odot \boldsymbol{x} \Vert_2^2}
 \end{equation}
     
Note that in the event $\boldsymbol{x}=0$, $t=0$, \textbf{Equation (\ref{tmp-derivation})} agrees with the fact that $eye(\boldsymbol{0})=0$. Thus \textbf{Equation (\ref{eye-defn})} and \textbf{(\ref{eye-defn-orig})} are equivalent.
\end{proof}

\subsection{Sparsity with Orthonormal Design Matrix}

We consider a special case of regression and orthogonal design matrix ($X^\top X = I$) with EYE regularization. This restriction allows us to obtain a closed form solution so that key features of EYE penalty can be highlighted. With \textbf{Equation (\ref{eye-defn})}, we have
  
     \begin{equation} \label{regression-obj}
     \resizebox{.47\textwidth}{!}{$
   \min_{\boldsymbol{\theta}} \frac{1}{2} \Vert \boldsymbol{y} - X \boldsymbol{\theta} \Vert_2^2 + n \lambda 
   \left ( \Vert (\boldsymbol{1}-\boldsymbol{r}) \odot \boldsymbol{\theta} \Vert_1 + 
   \sqrt{\Vert (\boldsymbol{1}-\boldsymbol{r}) \odot \boldsymbol{\theta} \Vert_1^2 + 
    \Vert \boldsymbol{r} \odot \boldsymbol{\theta} \Vert_2^2} \right )
    $}
   \end{equation}
   
   Since the objective is convex, we solve for its subgradient $\boldsymbol{g}$.
   
   \begin{equation}  \label{orthog-general}
   \resizebox{.47\textwidth}{!}{$
   \boldsymbol{g} = X^\top X \boldsymbol{\theta} - X^\top \boldsymbol{y} + n \lambda (\boldsymbol{1}-\boldsymbol{r}) \odot \boldsymbol{s} + 
   \frac{n\lambda}{Z} (\Vert (\boldsymbol{1}-\boldsymbol{r}) \odot \boldsymbol{\theta} \Vert_1 (\boldsymbol{1}-\boldsymbol{r}) \odot \boldsymbol{s} +  \boldsymbol{r} \odot \boldsymbol{r} \odot \boldsymbol{\theta})
   $}
   \end{equation}
   
   where $s_i = sgn(\theta_i)$ if $\theta_i \neq 0$, 
   $s_i \in [-1,1]$ if $\theta_i =0$, and
   $Z = \sqrt{\Vert (\boldsymbol{1}-\boldsymbol{r}) \odot \boldsymbol{\theta} \Vert_1^2 + 
    \Vert \boldsymbol{r} \odot \boldsymbol{\theta} \Vert_2^2}$.
   
   By our assumption $X^\top X = I$, and the fact that 
   $\boldsymbol{\hat \theta}^{OLS} = (X^\top X)^{-1} X^\top \boldsymbol{y} = X^\top \boldsymbol{y}$
   (the solution for oridinary least squares), we simplify (\ref{orthog-general}) as 

   \begin{equation}
   \resizebox{.47\textwidth}{!}{$
   \boldsymbol{g} = \boldsymbol{\theta} - \boldsymbol{\hat \theta}^{OLS} + n \lambda (\boldsymbol{1}-\boldsymbol{r}) \odot \boldsymbol{s} + 
   \frac{n \lambda}{Z} (\Vert (\boldsymbol{1}-\boldsymbol{r}) \odot \boldsymbol{\theta} \Vert_1 (\boldsymbol{1}-\boldsymbol{r}) \odot \boldsymbol{s} + \boldsymbol{r} \odot \boldsymbol{r} \odot \boldsymbol{\theta})
   $}
   \end{equation}
   
   setting $\boldsymbol{g}$ to $\boldsymbol{0}$ we have

   
   
   {\small
   \begin{equation} \label{orthog-theta}
   \hat \theta_i = \frac{\hat \theta_i^{OLS}}{1+\frac{n \lambda}{Z} r_i^2}
   \max \left ( 0, 
   1-\frac{n \lambda (1-r_i) \left (1+\frac{\Vert (\boldsymbol{1}-\boldsymbol{r}) \odot \boldsymbol{\hat \theta} \Vert_1}{Z} \right )}{\left |\hat \theta_i^{OLS}\right |} \right )
   \end{equation}
    }
    
   where $Z = \sqrt{\Vert (\boldsymbol{1}-\boldsymbol{r}) \odot \boldsymbol{\hat \theta} \Vert_1^2 + 
    \Vert \boldsymbol{r} \odot \boldsymbol{\hat \theta} \Vert_2^2}$.

   Note that \textbf{Equation (\ref{orthog-theta})} is still an implicit equation in $\boldsymbol{\theta}$
   because $Z$ is a function of $\boldsymbol{\hat \theta}$. Also, we implicitly assumed that $Z \neq 0$.
   
    Although this is an implicit equation for $\theta_i$, the max term confirms EYE's ability to set weights to exactly zero in the orthonormal design matrix setting.

	What if $Z=0$? This only happens if $\boldsymbol{\theta}=\boldsymbol{0}$. However, by the complementary slackness condition in KKT, we know $\lambda>0$ implies that the solution is on the boundary of the constraint formulation of the problem (for $\lambda=0$, we are back to ordinary least squares). So long as the optimal solution for the unconstrained problem is not at $\boldsymbol{0}$, we won't get into trouble unless the constraint is $eye(\boldsymbol{\theta}) \leq 0$, which won't happen in the regression setting as $\lambda$ is finite. If the optimal solution for the unconstrained problem is $\boldsymbol{0}$, we are again back to ordinary least squares solutions. So the upshot is we can assume $Z \neq 0$ otherwise it will automatically revert to ordinary least squares.
	
\subsection{Perfect Correlation}

Denote the objective function in \textbf{Equation (\ref{regression-obj})} as $L(\boldsymbol{\theta})$. Assume $\boldsymbol{\hat \theta}$ is the optimal solution, $x_i = x_j$  (\eg, the $i^{th}$ and $j^{th}$ columns of design matrix are co-linear)

\begin{itemize}
\item $r_i = 1$, $r_j = 0$, $x_i = x_j$ $\implies$ $\hat \theta_j = 0$

    Here, we show EYE penalty prefers known risk factors over unknown risk factors.

    \begin{proof}
    Assume $r_i=1$, $r_j=0$.
    
    consider $\boldsymbol{\hat \theta'}$ that only differs from $\boldsymbol{\hat \theta}$ at the $i^{th}$ and $j^{th}$ entry such that $\hat \theta'_i = \hat \theta_i + \hat \theta_j$ and $\hat \theta'_j=0$.
    
    $L(\boldsymbol{\hat \theta}) -L(\boldsymbol{\hat \theta'}) = \frac{1}{2} \Vert \boldsymbol{y}-X \boldsymbol{\hat \theta}\Vert_2^2 +  n\lambda \left ( | \hat \theta_j | + \sqrt{(C+  | \hat \theta_j |)^2 + D + \hat \theta_i^2} \right)  - \frac{1}{2} \Vert \boldsymbol{y}-X \boldsymbol{\hat \theta'} \Vert_2^2 - n\lambda \left( |\hat \theta_j' | + \sqrt{(C+ |\hat \theta_j' |)^2 + D + \hat \theta_i'^2} \right )$
    
    where $C$ and $D$ are non-negative constant involving entries other than $i$ and $j$. Note that the sum of squared residual is the same for both $\boldsymbol{\hat \theta'}$ and $\boldsymbol{\hat \theta}$ owing to the fact that $x_i=x_j$.Use the definition of $\boldsymbol{\hat \theta'}$, we have
    
    \begin{equation*}
    \resizebox{.465\textwidth}{!}{$
    L(\boldsymbol{\hat \theta}) - L(\boldsymbol{\hat \theta'}) = n\lambda \left ( | \hat \theta_j | + \sqrt{(C+ | \hat \theta_j |)^2 + D + \hat \theta_i^2}  - \sqrt{C^2 + D + (\hat \theta_i + \hat \theta_j)^2} \right )
    $}
    \end{equation*}
    
    \begin{claim} \label{claim1}
    $L(\boldsymbol{\hat \theta}) - L(\boldsymbol{\hat \theta'}) \geq 0$ with equality only if $\hat \theta_j=0$
    \end{claim}
    
    \begin{proof}
    Since $n\lambda$ is positive, the claim is equivalent to
    $$ \sqrt{(C+ | \hat \theta_j |)^2 + D + \hat \theta_i^2} 
    \geq \sqrt{C^2 + D + (\hat \theta_i + \hat \theta_j)^2} - | \hat \theta_j |$$
    
    If the right hand side is negative, we are done since the left hand side is non-negative.
    
    Otherwise, both sides are non-negative. We square them and rearrange to get the equivalent form
    
    $$\hat \theta_j^2 + 2 \hat \theta_i \hat \theta_j \leq 2 |\hat \theta_j| \sqrt{C^2+D+(\hat \theta_i + \hat \theta_j)^2} + 2 C |\hat \theta_j|$$
    
    which is true following
  
    \begin{align}
    \hat \theta_j^2 + 2 \hat \theta_i \hat \theta_j &\leq 2\hat \theta_j^2 + 2 \hat \theta_i \hat \theta_j - \hat \theta_j^2 \label{sq-drop1}\\
    &\leq 2|\hat \theta_j||\hat \theta_i+\hat \theta_j| \label{sq-drop2} \\
    &= 2|\hat \theta_j|\sqrt{(\hat \theta_i+\hat \theta_j)^2}\\
    &\leq 2 |\hat \theta_j| \sqrt{C^2+D+(\hat \theta_i + \hat \theta_j)^2} + 2 C |\hat \theta_j|
    \end{align}
    
    Again if $\hat \theta_j \neq 0$, the inequality is strict from \textbf{Equation (\ref{sq-drop1})} to \textbf{Equation (\ref{sq-drop2})}
  
    \end{proof}
  
    Since we assumed that $\boldsymbol{\hat \theta}$ is optimal, the equality in \ref{claim1} must hold, thus $\hat \theta_j=0$.
  
    \end{proof}

\item $r_i = 1$, $r_j = 1$, $x_i = x_j$ $\implies$ $\hat \theta_i = \hat \theta_j$

    Feature weights are dense in known risk factors

   \begin{proof}
    Assume $\boldsymbol{\hat \theta}$ is optimal, consider $\boldsymbol{\hat \theta'}$ that is the same as $\boldsymbol{\hat \theta}$ except $\hat \theta'_i = \hat \theta'_j= \frac{\hat \theta_j + \hat \theta_j}{2}$.

    Assume $\boldsymbol{\hat \theta} \neq \boldsymbol{\hat \theta'}$: $\hat \theta_i \neq \hat \theta_j$.  Again the sum of residue of for both estimation is unchanged as $x_i=x_j$
    
    \begin{equation*}
    \resizebox{.48\textwidth}{!}{$
    L(\boldsymbol{\hat \theta}) - L(\boldsymbol{\hat \theta'}) = n\lambda \left ( \sqrt{\left(C+|\hat \theta_i|+|\hat \theta_j|\right)^2 +D+\hat \theta_i^2 + \hat \theta_j^2} - \sqrt{\left (C+2\frac{|\hat \theta_i + \hat \theta_j|}{2} \right )^2 + D+ 2 \frac{|\hat \theta_i + \hat \theta_j|^2}{4}} \right)
    $}
    \end{equation*}
    
    which is greater or equal to 
    \begin{equation*}
        \resizebox{.47\textwidth}{!}{$
        n\lambda \left ( \sqrt{ \left (C+|\hat \theta_i|+|\hat \theta_j| \right )^2 +D+\hat \theta_i^2 + \hat \theta_j^2} - \sqrt{\left (C+|\hat \theta_i| + |\hat \theta_j| \right)^2 + D+ \frac{|\hat \theta_i + \hat \theta_j|^2}{2}} \right)
        $}
    \end{equation*}


    Since $$\hat \theta_i^2 + \theta_j^2 - \frac{|\hat \theta_i + \hat \theta_j|^2}{2} = \frac{(\hat \theta_i - \hat \theta_j)^2}{2}>0$$ by assumption that $\hat \theta_i \neq \hat \theta_j$ for the optimal solution. This shows $L(\boldsymbol{\hat \theta}) - L(\boldsymbol{\hat \theta'})>0$, which contradict our assumption.
    
    Thus $\hat \theta_i=\hat \theta_j$ for the optimal solution.
    \end{proof}

\item  $r_i = 0$, $r_j = 0$, $x_i = x_j$ $\implies$ back to LASSO continuum    

    Note that fixing $\theta_k$ $\forall k \not \in \{i,j\}$, solving for $\theta_i$ and $\theta_j$ reduces the problem  to LASSO, thus all properties of LASSO carry over for $\theta_i$ and $\theta_j$. Thus sparsity is maintained in unknown features.
    

    


\end{itemize}
 
 \subsection{General Correlation}

  Grouping effect in elastic net is still present in eye penalty within groups with similar level of risk.

 \begin{theorem}
  if $\hat \theta_i \hat \theta_j > 0$ and design matrix is standardized, then
    \begin{equation*}
    \resizebox{.48\textwidth}{!}{$
  \frac{|r_i^2 \hat \theta_i - r_j^2 \hat \theta_j|}{Z} \leq \frac{\sqrt{2 (1-\rho)} \Vert \boldsymbol{y} \Vert_2}{n\lambda}
  + |r_i-r_j| \left (1+\frac{\Vert (\boldsymbol{1}-\boldsymbol{r}) \odot \hat \theta \Vert_1}{Z} \right )
  $}
  \end{equation*}

  where $Z = \sqrt{\Vert (\boldsymbol{1}-\boldsymbol{r}) \odot \boldsymbol{\hat \theta} \Vert_1^2 + \Vert \boldsymbol{r} \odot \boldsymbol{\hat \theta} \Vert_2^2}$, $\rho$ is the sample covariance between $x_i$ and $x_j$
  \end{theorem}

  \begin{proof}
  Denote the objective in \textbf{Equation (\ref{regression-obj})} as $L$. Assume $\hat \theta_i \hat \theta_j > 0$, $\boldsymbol{\hat \theta}$ is the optimal weights, and the design matrix $X$ is standardized to have zero mean and unit variance in its column. Via the optimal condition and (\ref{orthog-general}), subgradient $\boldsymbol{g}$ at $\boldsymbol{\hat \theta}$ is $0$. Hence we have
  
  \begin{equation} \label{corr-eq1}
  \resizebox{.48\textwidth}{!}{$
  -x_i^\top(\boldsymbol{y}-X \boldsymbol{\hat \theta}) + n\lambda((1-r_i) s_i + \frac{\Vert (\boldsymbol{1}-\boldsymbol{r}) \odot \boldsymbol{\hat \theta} \Vert_1}{Z}
  ((1-r_i) s_i + r_i^2 \hat \theta_i)) = 0
  $}
  \end{equation}
  
  \begin{equation} \label{corr-eq2}
  \resizebox{.48\textwidth}{!}{$
  -x_j^\top(\boldsymbol{y}-X\boldsymbol{\hat \theta}) + n\lambda((1-r_j) s_j + \frac{\Vert (\boldsymbol{1}-\boldsymbol{r}) \odot \boldsymbol{\hat \theta} \Vert_1}{Z}
  ((1-r_j) s_j + r_j^2 \hat \theta_j)) = 0 
  $}
  \end{equation}

  Substract \ref{corr-eq2} from \ref{corr-eq1}. The assumption that $\hat \theta_i \hat \theta_j > 0$ implies $sgn(\hat \theta_i)=sgn(\hat \theta_j)$ and eliminates the need to discuss the subgradient issue.
  
  \small{
  $(x_j^\top-x_i^\top)(\boldsymbol{y}-X \boldsymbol{\hat \theta}) + n\lambda((r_j-r_i)sgn(\hat \theta_i) + \frac{\Vert (\boldsymbol{1}-\boldsymbol{r}) \odot \boldsymbol{\hat \theta} \Vert_1}{Z}
  ((r_j-r_i) sgn(\hat \theta_i) + r_i^2 \hat \theta_i - r_j^2 \hat \theta_j)) = 0$
  }
  

  Rearrange to get

  \begin{equation} \label{corr-eq3}
  \resizebox{.48\textwidth}{!}{$
  \frac{r_i^2 \hat \theta_i - r_j^2 \hat \theta_j}{Z} = \frac{(x_i^\top - x_j^\top)(\boldsymbol{y}-X\hat \theta)}{n \lambda}
  + (r_i-r_j)sgn(\hat \theta_i) \left ( 1+\frac{\Vert (\boldsymbol{1}-\boldsymbol{r}) \odot \boldsymbol{\hat \theta} \Vert_1}{Z} \right )
  $}
  \end{equation}
  
  Being the optimal weights, $L(\boldsymbol{\hat \theta}) \leq L(\boldsymbol{0})$, which implies $\Vert y-X \boldsymbol{\hat \theta} \Vert_2^2 \leq \Vert \boldsymbol{y} \Vert_2^2$

  Also, standardized design matrix gives $\Vert x_i-x_j\Vert_2^2= \langle x_i, x_i \rangle + \langle x_j, x_j \rangle - 2 \langle x_i, x_j \rangle =2(1-\rho)$
  
  Taking the absolute value of \textbf{Equation (\ref{corr-eq3})} and applying Cauchy Schwarz inequality, we get
  
  \begin{equation}
  \resizebox{.48\textwidth}{!}{$
  \frac{|r_i^2 \hat \theta_i - r_j^2 \hat \theta_j|}{Z} \leq \frac{\Vert x_i - x_j \Vert_2 \Vert \boldsymbol{y}-X \boldsymbol{\hat \theta} \Vert_2}{n \lambda} + |r_i - r_j| \left (1+\frac{\Vert (\boldsymbol{1}-\boldsymbol{r}) \odot \boldsymbol{\hat \theta} \Vert_1}{Z} \right )
  $}
  \end{equation}
  
  which is less or equal to 
  
  \begin{equation}
  \resizebox{.3\textwidth}{!}{$
     \frac{\sqrt{2(1-\rho)}\Vert \boldsymbol{y} \Vert_2}{n \lambda}
  + |r_i - r_j| \left (1+\frac{\Vert (\boldsymbol{1}-\boldsymbol{r}) \odot \boldsymbol{\hat \theta} \Vert_1}{Z} \right )
  $}
  \end{equation}
  
  
  \end{proof}

  \begin{corollary}
  if $\hat \theta_i \hat \theta_j > 0$, design matrix is standardized, and $r_i=r_j \neq 0$
  $$\frac{|\hat \theta_i - \hat \theta_j|}{Z} \leq \frac{\sqrt{2(1-\rho)} \Vert \boldsymbol{y} \Vert_2}{r_i^2 n \lambda}$$

  where $Z = \sqrt{\Vert (\boldsymbol{1}-\boldsymbol{r}) \odot \boldsymbol{\hat \theta} \Vert_1^2 + \Vert \boldsymbol{r} \odot \boldsymbol{\hat \theta} \Vert_2^2}$,
  $\rho$ is the sample covariance between $x_i$ and $x_j$
  \end{corollary}
   
  This verifies the existence of the grouping effect: highly correlated features (with similar risk) are grouped together in the parameter space.

\end{document}